\documentclass[10pt,twocolumn,letterpaper]{article}

\usepackage{cvpr}
\usepackage{times}
\usepackage{epsfig}
\usepackage{graphicx}
\usepackage{amsmath}
\usepackage{amssymb}

\usepackage{macro}
\usepackage{float}
\usepackage{subfig}
\usepackage{multirow}
\usepackage[table,xcdraw]{xcolor}

\usepackage[pagebackref=true,breaklinks=true,letterpaper=true,colorlinks,bookmarks=false]{hyperref}

\cvprfinalcopy 

\def\cvprPaperID{3884} 

\ifcvprfinal\fi
\begin{document}

\title{Take it in your stride: Do we need striding in CNNs?}

\author{Chen Kong \hspace{2cm} Simon Lucey\\
Carnegie Mellon University\\
{\tt\small \{chenk, slucey\}@andrew.cmu.edu}
}

\maketitle

\begin{abstract}
Since their inception, CNNs have utilized some type of striding
operator to reduce the overlap of receptive fields and spatial
dimensions. Although having clear heuristic motivations (i.e. lowering
the number of parameters to learn) the mathematical role of striding
within CNN learning remains unclear. This paper offers a novel and
mathematical rigorous perspective on the role of the striding operator
within modern CNNs. Specifically, we demonstrate theoretically that one can always
represent a CNN that incorporates striding with an equivalent
non-striding CNN which has more filters and smaller size. Through this
equivalence we are then able to characterize striding as an additional
mechanism for parameter sharing among channels, thus reducing training
complexity. Finally, the framework presented in this paper offers a
new mathematical perspective on the role of striding which we
hope shall facilitate and simplify the future theoretical analysis of
CNNs.  
\end{abstract}

\section{Introduction}
Convolutional Neural Networks (CNNs)~\cite{lecun1990handwritten,
  lecun1998gradient, krizhevsky2012imagenet} have facilitated a dramatic
increase in the performance of perceptual tasks throughout various fields including image and signal processing~\cite{gatys2015neural, ulyanov2016texture, johnson2016perceptual, dong2016image}, speech recognition~\cite{bengio2003neural, hinton2012deep, mikolov2013cient}, and computer vision~\cite{farabet2013learning, simonyan2014very, he2016deep}.
Almost all CNNs used in the above applications employ some sort of
striding operator, but to date there has been limited theoretical
analysis of its role other than as a heuristic for lowering the
degrees of freedom within the network. For the purposes of this paper
we shall refer to striding in the context of convolution, where the
stride refers to the relative offset applied to filter kernel. Classical convolution implies a stride of one,
however, non-unity stride values are commonly entertained within CNN
literature (see Figure~\ref{fig:teaser} for
a visualization). Herein, when we state that a
CNN does not employ striding we are actually implying that it is using
only convolutional operators with unity stride. 

The role of striding has been championed within CNN
architectures as: (i) it can reduce spatial resolution, leading to
computational benefits; and (ii) can reduce the overlap of receptive
fields. Even though these two explanations provide some motivations to
a certain degree, they are still largely superficial. 

\begin{figure}[t]
    \centering
    \includegraphics[width=\linewidth]{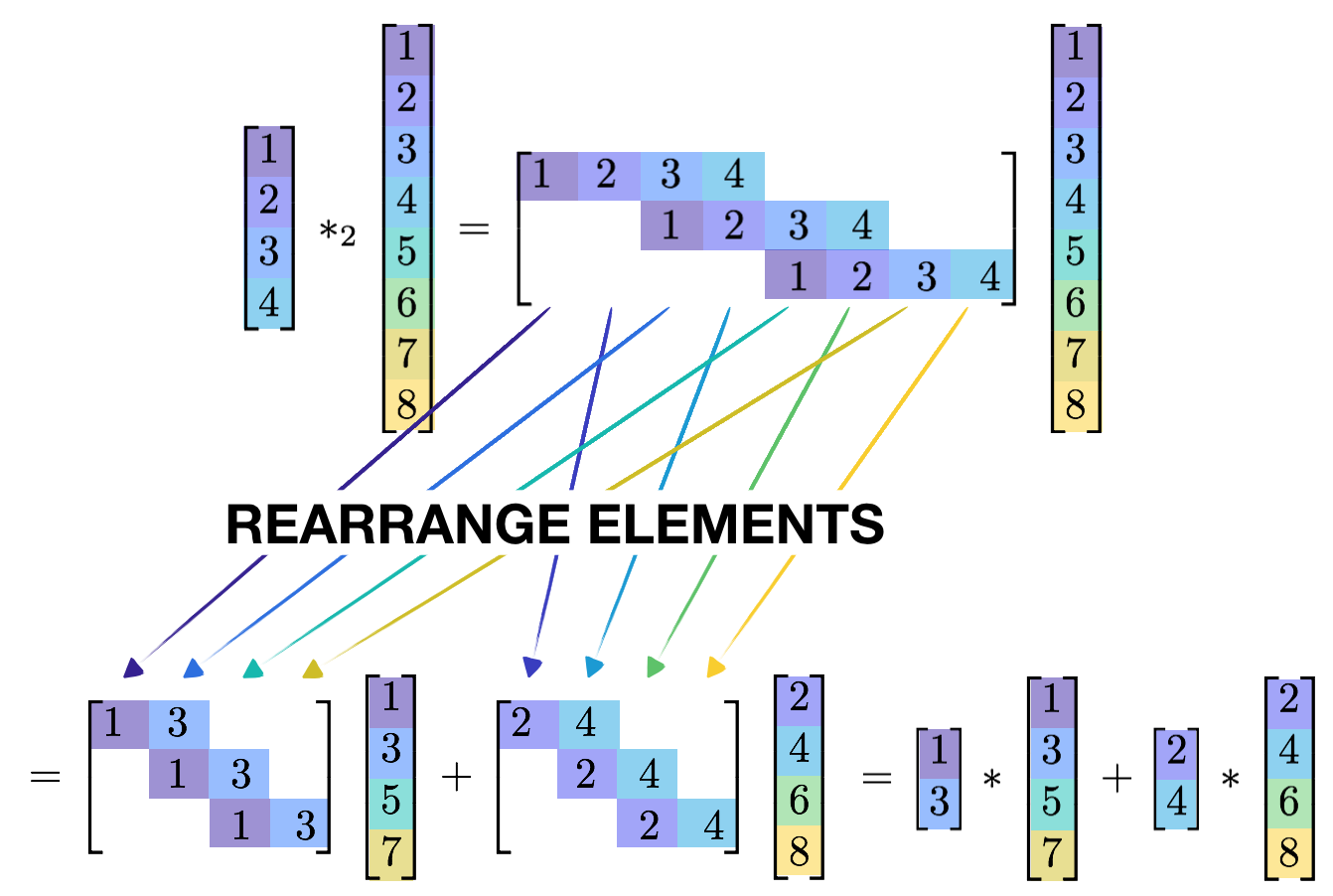}
    \caption{A toy example to show how 1-D convolution of stride two
      can be reduced to stride one by simply rearranging the elements
      in both filters and signal. The operator $*_2$ denotes a
      convolution operator with stride of two. One can see by
      rearranging the columns in the banded strided Toeplitz matrix,
      the convolution of stride two is equivalent to the summation of
      the conventional convolutional operator with stride one.}
    \label{fig:teaser}
\end{figure}

In this paper we offer a new mathematical tool to characterize
theoretically the role of striding within modern CNNs, and why their
employment has been so crucial for the empirical success of these
networks across a myriad of perceptual tasks. To facilitate this
characterization we ask a fundamental question: is striding necessary
within a CNN?  The short answer is \textbf{no}, when evaluating a
pre-trained model. This paper demonstrates that any
feedforward CNN utilizing non-unity strided convolution can be
evaluated equivalently as a feedforward CNN that employs only unity strided
convolution. This claim is based on a simple yet elegant observation
that any non-unity strided convolution can be equivalently simplified to
a classical convolution with unity stride, but with the additional filters of
smaller size. Figure~\ref{fig:teaser} visualizes this insight through a toy example.
One can see that the convolution between vectors $[1,2,3,4]^T$ and
$[1,2,3,4,5,6,7,8]^T$ with stride two is equivalent to the summation
of two convolutions with stride one. 
If one considers each smaller filter as a channel, then the summation of two convolutions can be reinterpreted as a two-channel convolution with unity stride.
This insight is at the heart of our paper.

Striding, however, is still useful. We further demonstrate that our
proposed simplification strategy actually increases the parameter
space of the CNN implying an increase in the capacity of the
network. Therefore during training, the striding operator reduces the
parameter space by forcing parameter sharing among different channels
and potentially helping the generalization properties of the network. 

\noindent\textbf{Contributions:} We make the following contributions:
\begin{itemize}
    \item We establish a clear mathematical definition of strided
      convolution and unveil an equivalence between multi-stride and
      multi-channel convolutions. 
    \item Theoretically demonstrate that any feed forward CNN
      employing striding has a mathematically equivalent non-striding
      CNN architecture during evaluation.  
    \item We reinterpret striding as a tool for sharing parameters
      along channels, and argue that this connection gives a more
      thorough and theoretical explanation for why 
      striding is still an invaluable tool when designing CNN
      architectures. 
\end{itemize}

\section{Related Work}
In the history of CNNs, some architectures, \eg AlexNet~\cite{krizhevsky2012imagenet}, ZFNet~\cite{zeiler2014visualizing}, GoogLeNet~\cite{szegedy2015going}, ResNet~\cite{he2016deep}, \etc, utilize large filters with non-unity stride, while some architectures, \eg VGG net~\cite{simonyan2014very} and more, utilize smaller filters with unity stride.
Despite of such common usage of striding, none of them explains the reasons for their designed stride size in a clear mathematical way.
This makes striding most like a heuristic choice.

Recently, Springenberg~\etal~\cite{springenberg2014striving} were
devoted to pursuing simpler CNN architectures and questioned the
necessity of different components in the canonical pipeline.
They found that max-pooling can simply be replaced by a convolutional layer with increased stride
without loss in accuracy on several image recognition benchmarks.
Based on this finding, they proposed a novel architecture that
consists solely of convolution layers utilizing ReLU as the sole non-linearity.
Although their work does not focus on the role of striding, the proposed replacement implies that striding possibly serves similar functionality to a certain type of pooling.
Further, due to the linear essence of convolutions, convolutions with large stride is more preferable.
This offers a novel perspective to understanding striding.

More recently, Papyan~\etal~\cite{papyan2016convolutional} proposed to reinterpret the forward pass of CNNs as a thresholding pursuit of signals modeled through a novel Multi-Layer Convolutional Sparse Coding(ML-CSC) model.
This reinterpretation gave a clear mathematical meaning, objective and model to CNNs, which can be in turn used to analyze the guarantees for the success of the forward pass.
Specifically, the mutual coherence of a learned convolutional
dictionary serves a major role in deciding the uniqueness of recovery of signals in ML-CSC model, and thus affecting the success of the forward pass.
Based on this theory, they proposed that strides not only bring
computational benefit, but also some theoretical benefits in terms of
guarantees on uniqueness.
On the one hand, striding lowers the mutual coherence of convolutional
dictionary, thus leading to more non-zeros allowed per stripes. On the other hand, strides decrease the length of stripes.
These twofold together encourage a sparse solutions to the ML-CSC
model and thus a higher possibility of success that the forward pass
will generate unique codes. However, due to the lack of sufficient
experimental support, their interpretation of CNNs remains a
theoretical hypothesis.

\begin{figure}[t]
    \captionsetup[subfigure]{labelformat=empty}
    \begin{tabular}{cc}
        \centering
        \multirow{1}{*}[3.5em]{\subfloat[{\normalsize (a) $\Hv$}]{\includegraphics[width=0.145\linewidth]{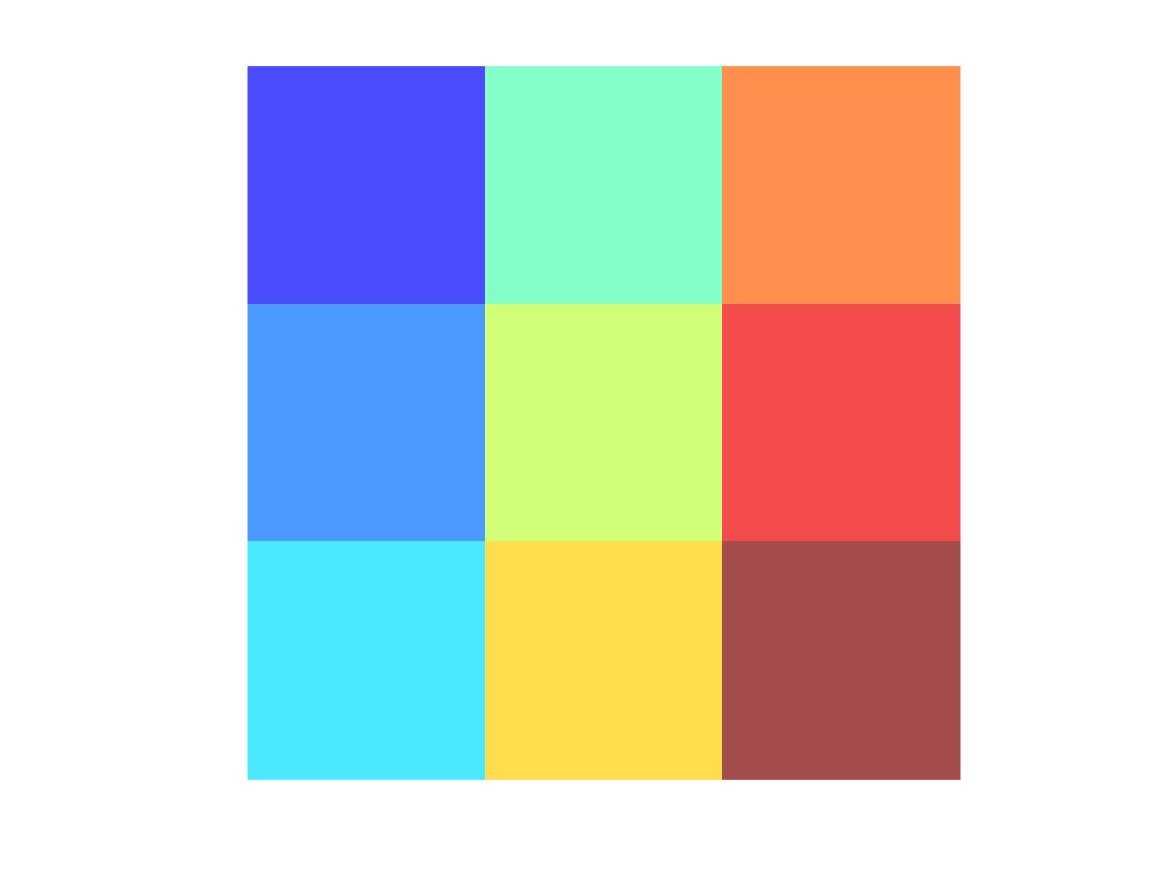}}} &
        \hspace{-4mm}
        \subfloat[$\Hc(1, 1, :, :)$]{\includegraphics[width=0.25\linewidth]{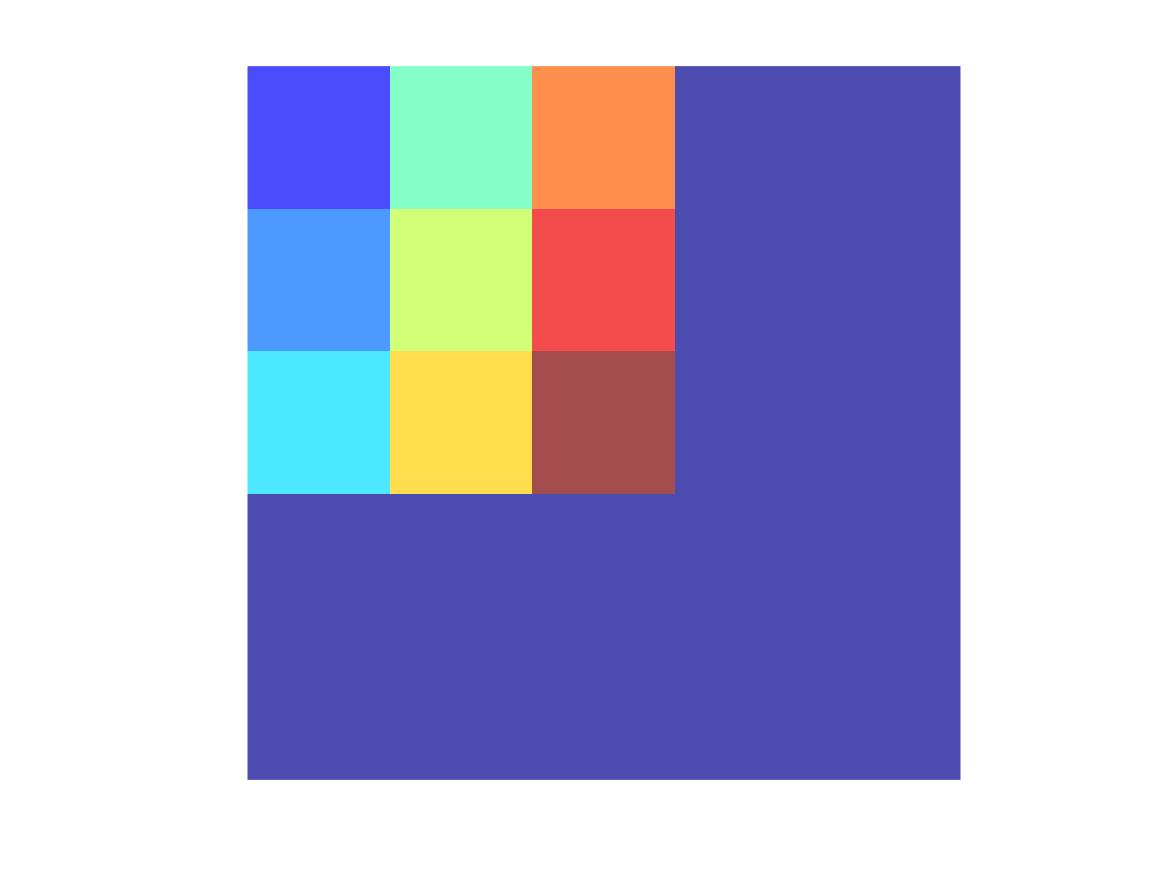}}
        \subfloat[$\Hc(1, 2, :, :)$]{\includegraphics[width=0.25\linewidth]{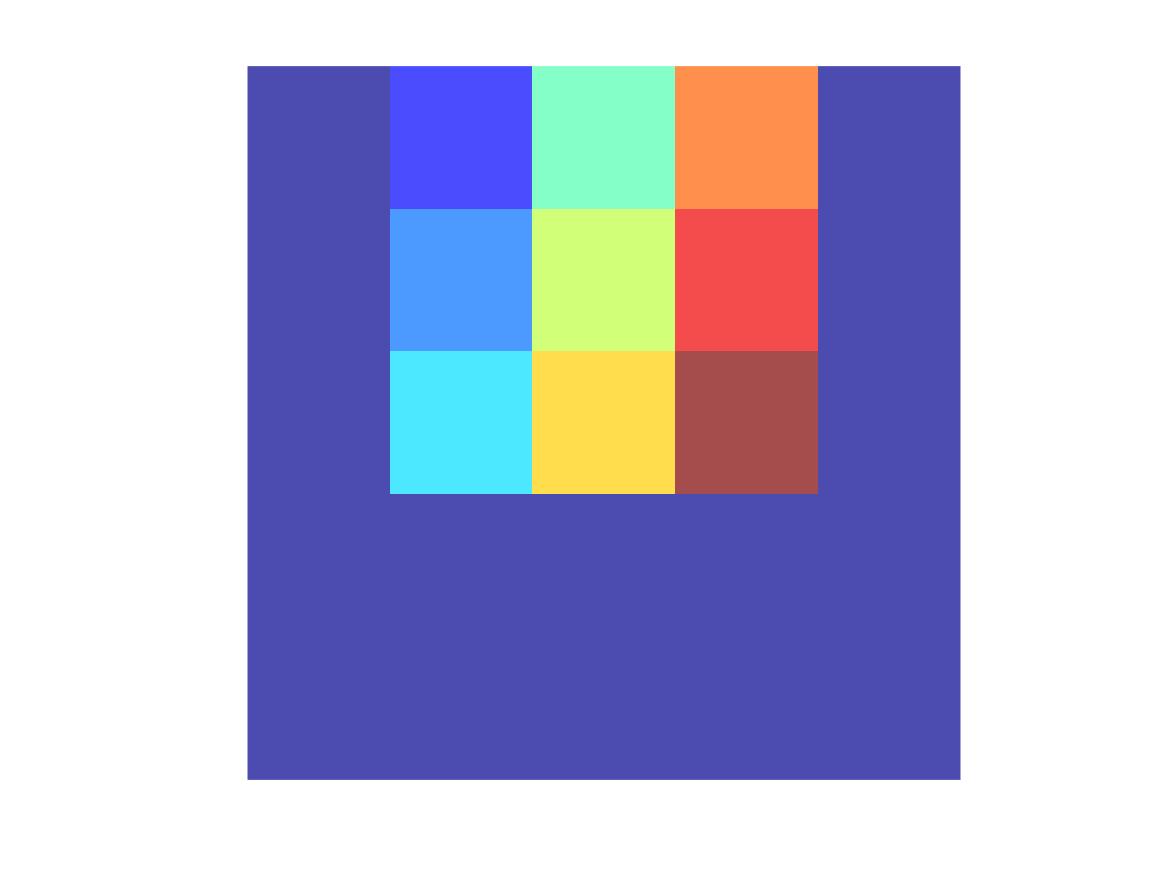}}
        \subfloat[$\Hc(1, 3, :, :)$]{\includegraphics[width=0.25\linewidth]{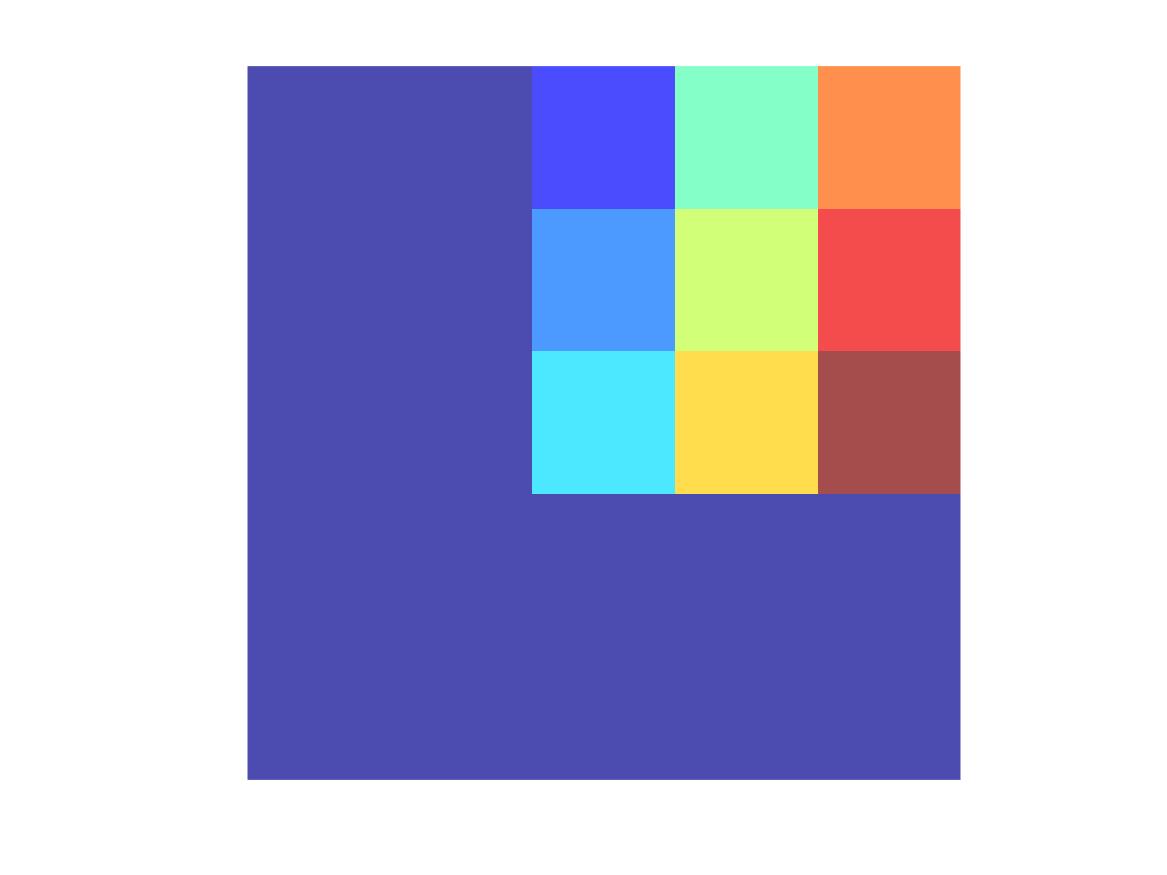}}\\

        \multirow{2}{*}[3em]{\includegraphics[width=0.13\linewidth]{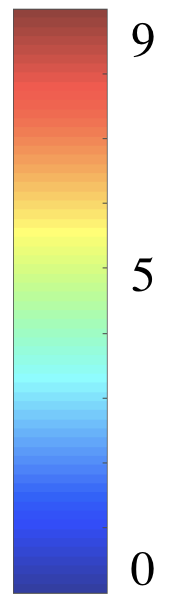}} &
        \hspace{-4mm}
        \subfloat[$\Hc(2, 1, :, :)$]{\includegraphics[width=0.25\linewidth]{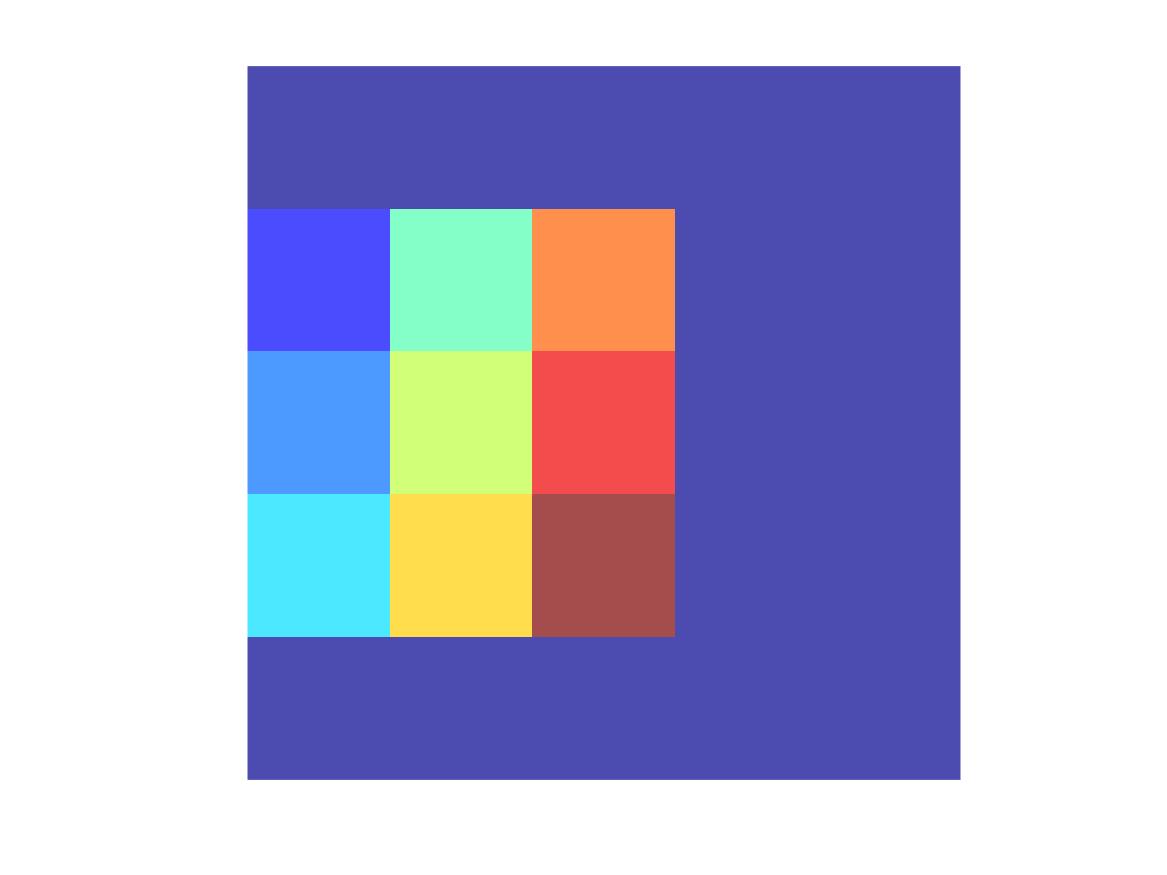}}
        \subfloat[$\Hc(2, 2, :, :)$]{\includegraphics[width=0.25\linewidth]{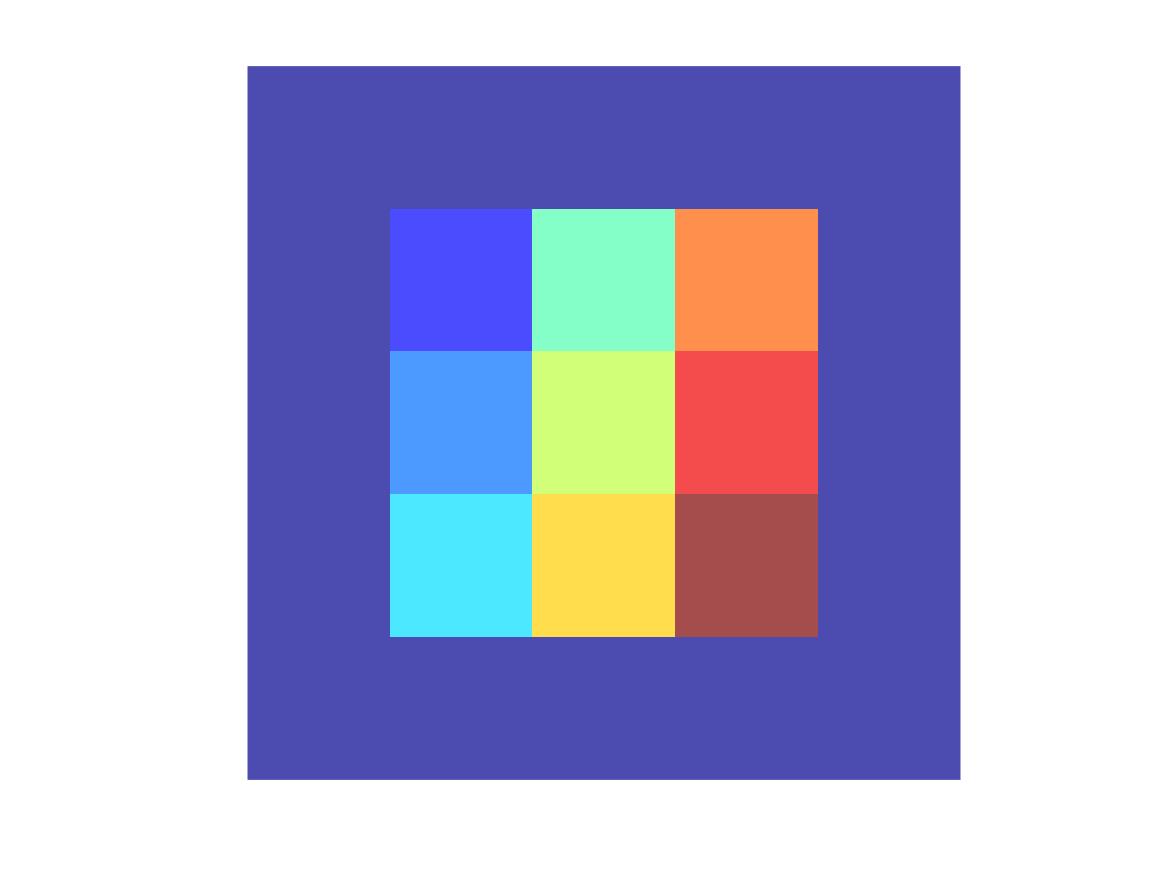}}
        \subfloat[$\Hc(2, 3, :, :)$]{\includegraphics[width=0.25\linewidth]{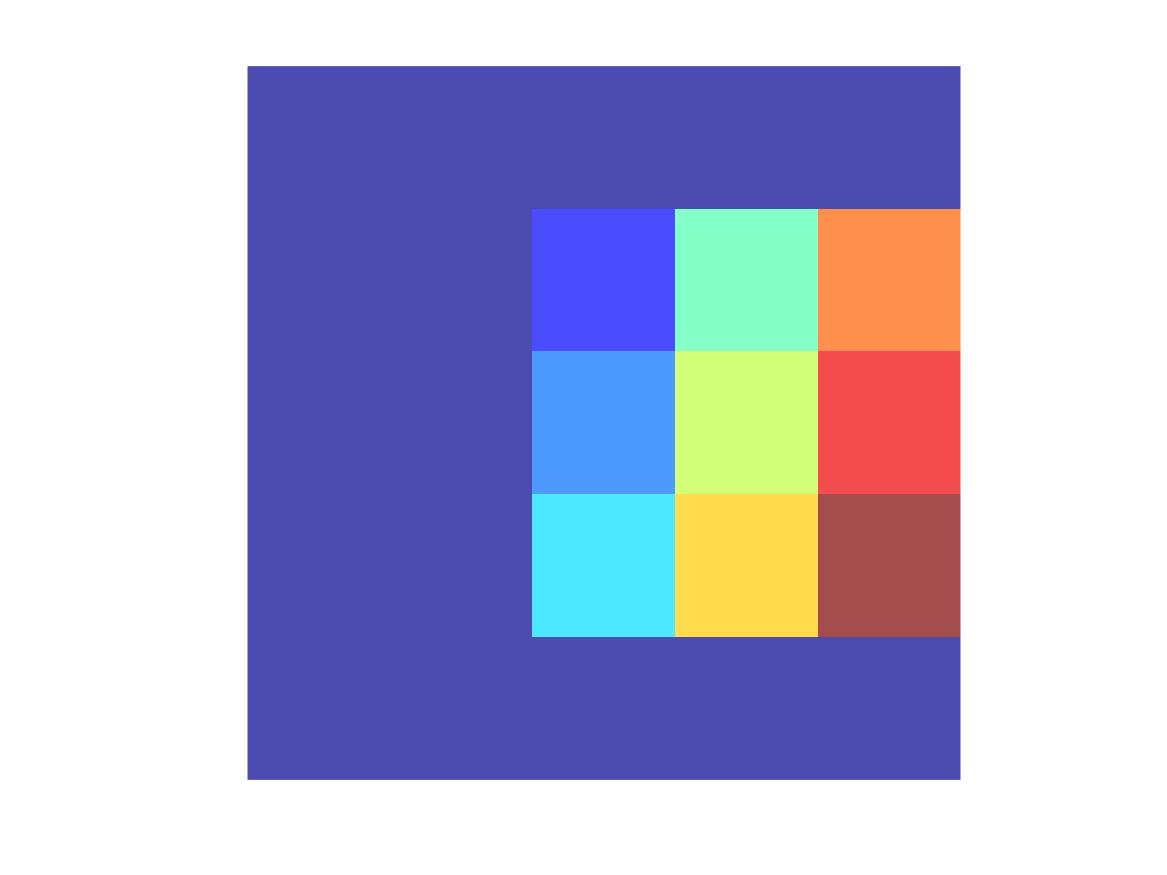}}\\

        &
        \hspace{-4mm}
        \subfloat[$\Hc(3, 1, :, :)$]{\includegraphics[width=0.25\linewidth]{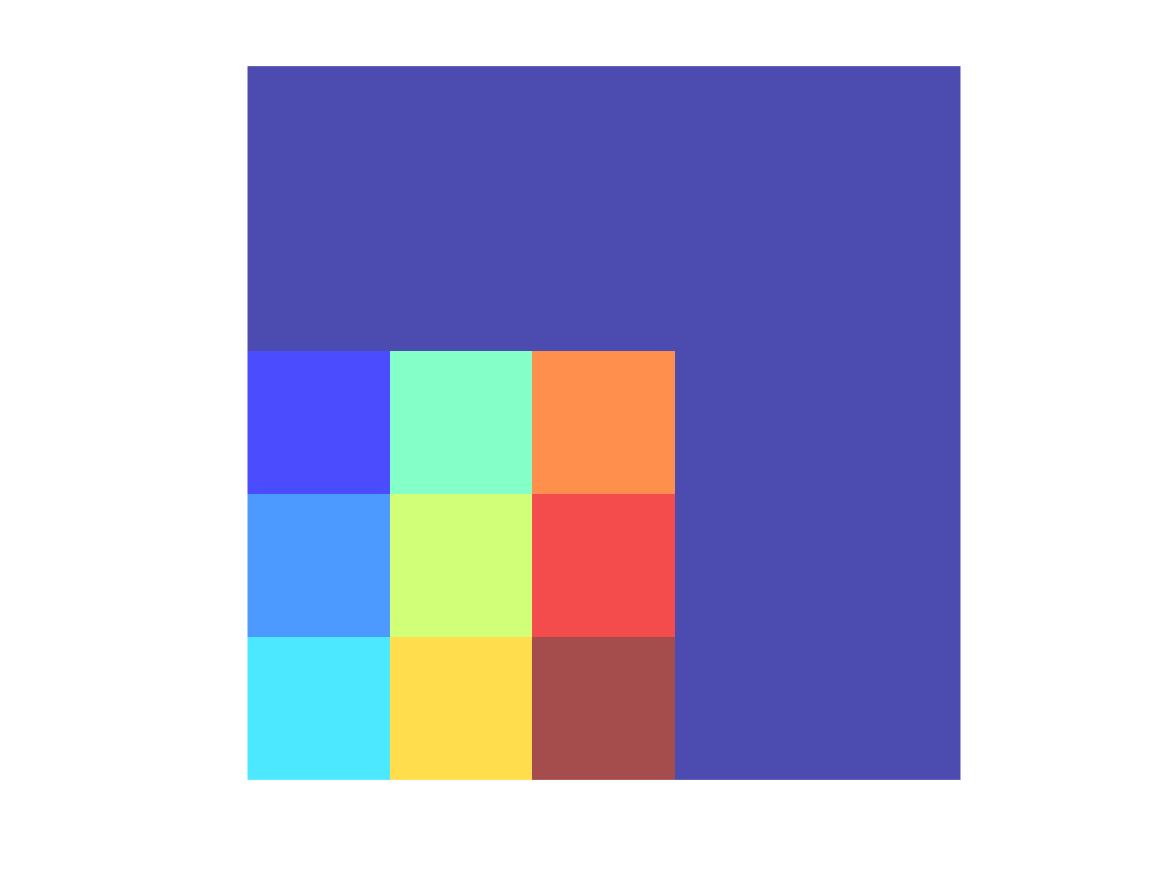}}
        \subfloat[$\Hc(3, 2, :, :)$]{\includegraphics[width=0.25\linewidth]{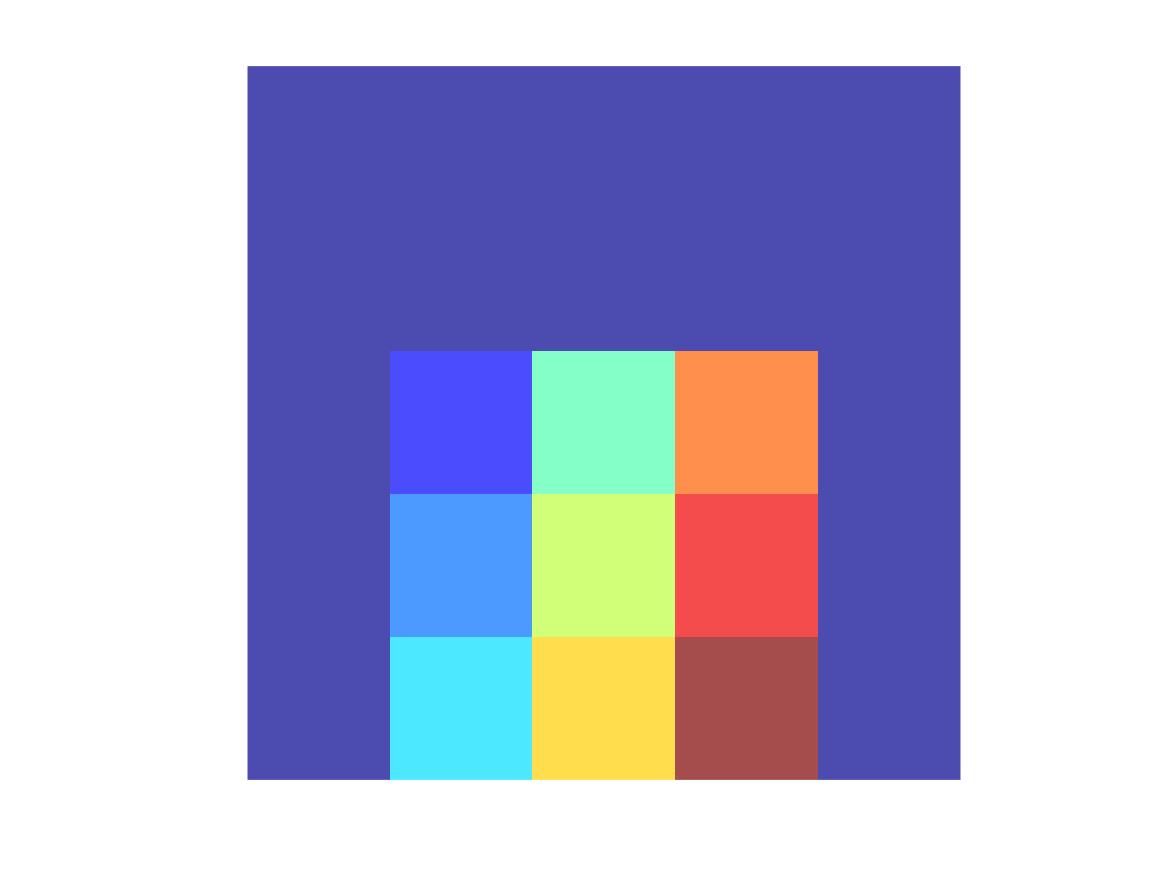}}
        \subfloat[$\Hc(3, 3, :, :)$]{\includegraphics[width=0.25\linewidth]{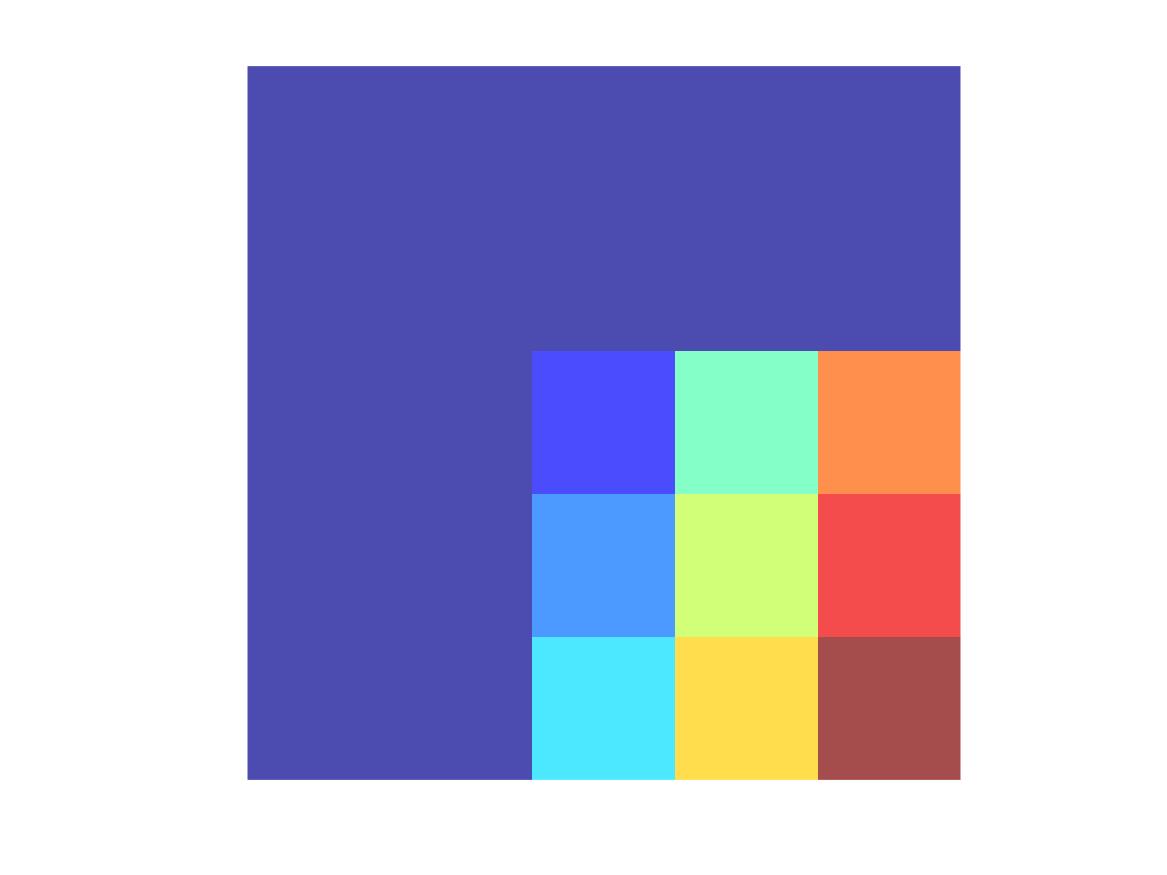}}\\

        (b) Colorbar & 
        \hspace{-4mm}
        (c) Convolutional Tensor $\Hc$
    \end{tabular}
    \caption{An example of convolutional tensors. $\Hc$ is a $3\times3\times5\times5$ convolutinoal tensor corresponding to the convolution between the filter $\Hv$ and a $5\times5$ image.}
    \label{fig:ex-conv-tensor}
\end{figure}

\section{Prerequisite: Convolution as Tensor Product}
Before considering strides, as prerequisite we first define tensor product and use it to represent convolutions.
Convolution as a linear operator can always be represented as a matrix or tensor multiplication.
For example, in one-dimensional space, convolution of two vectors $\hv$ and $\xv$ are known to equal to the multiplication of a Toeplitz matrix generated by $\hv$ and the vector $\xv$.
Now, we show this equivalence also holds for two-dimensional space.
\begin{definition}
    \label{def:tensor-product}
    (Tensor product):
    Given a four-dimensional tensor $\Hc\in\RR^{m\times n\times d \times c}$ and a matrix $\Xv\in\RR^{c \times d}$, the product of $\Hc$ and $\Xv$ is a matrix with dimension $m\times n$.
    Formally,
    \begin{equation}
        \Hc\Xv \in \RR^{m\times n},~\st~[\Hc\Xv]_{i, j} = \sum_{k=1}^d \sum_{\ell=1}^c \Hc_{i,j,k,\ell}\Xv_{\ell, k}.
    \end{equation}
\end{definition}

Further, inspired by Toeplitz matrices, we define a certain set of four-dimensional tensors that share a specific structured pattern.
\begin{definition}
    \label{def:conv-tensor}
    (Convolutional Tensor):
    Given an image filter $\Hv\in\RR^{a\times b}$, the convolutional tensor $\Hc$ corresponding to $\Hv$ is defined as a four-dimensional tensor such that
    \begin{equation}
        \Hc(i, j, i:i+a-1, j:j+b-1) = \Hv,
    \end{equation}
    where $i, j \in \ZZ^+$.
    The operator $:$ denotes tensor slicing.
\end{definition}

Similar to a Toeplitz matrix, a convolutional tensor actually covers all possible shifted version of the corresponding image filter, shown in Figure~\ref{fig:ex-conv-tensor}.
One may also notice that the size of a convolutional tensor depends not only on the size of corresponding image filter, but also on the size of convolved image.

From the definition of the convolutional tensor, one can derive the following three important properties.
For the sake of brevity, we omit the proof of them.

\begin{property}
    \label{propty:ele}
    A convolutional tensor $\Hc$ must satisfy
    \begin{equation}
            \Hc_{i, j, k, \ell} = \Hc_{i+1, j, k+1, \ell} = \Hc_{i, j+1, k, \ell+1},
    \end{equation}
    for any valid $i, j, k, \ell$.
\end{property}

\begin{property}
    \label{propty:conv}
    Any four-dimensional tensor $\Hc$ satisfying
    \begin{equation}
            \Hc_{i, j, k, \ell} = \Hc_{i+1, j, k+1, \ell} = \Hc_{i, j+1, k, \ell+1},
            \label{eq:shift}
    \end{equation}
    for any valid $i, j, k, \ell$, must be a convolutional tensor.
\end{property}

This property is of important since it provides an additional support to characterize convolutional tensors.
Note that Property~\ref{propty:ele} and~\ref{propty:conv} together introduce the equality constraints in Equation~\ref{eq:shift} as a sufficient and necessary condition to a convolutional tensor.

\begin{property}
    (Convolution as Tensor Product):
    Convolution of an image filter $\Hv\in\RR^{a\times b}$ and an image $\Xv$ equals to the tensor product of the corresponding convolutional tensor $\Hc$ and the image $\Xv$.
    Formally
    \begin{equation}
        \Hv * \Xv = \Hc \Xv.
    \end{equation}
\end{property}

This property depicts the major role of a convolutional tensor, representing a convolution as a tensor product.
One can prove this property trivially from the definition of convolutions and Definition~\ref{def:conv-tensor}.
One advantage of representing a convolution as a tensor product is to allow us to analyze the parameters of convolutions, \eg stride, pad, \etc,  by manipulating the corresponding convolutional tensor.
Specifically, utilizing non-unity stride in convolution is equivalent to applying a sub-sampling function to the associated convolutional tensor.
Before digging into details, we need one more tool to help, which is a sub-sampling function.

\begin{definition}
    \label{def:sub-sampling-matrix}
    (Sampling function of matrix)
    Define $\Tc_{m, n}^{(s)}(\cdot)$ as a function with a matrix as input and its down-sampled sub-matrix as output.
    This function select the elements on the grid defined by $m,n,s$, where $(m, n)$ denotes the upper-left position of the grid, and $s$ denotes the sub-sampling stride.
    Note that $m, n \le s$.
\end{definition}

To make this definition clearer, Figure~\ref{fig:ex-S} shows an exampling of $\Tc_{m,n}^{(2)}(\cdot)$ applied to a $6\times6$ matrix, for $m = 1, 2$ and $n = 1, 2$ respectively.
From this example, one might notice that the elements in a matrix will be sampled and only sampled once by the proposed sub-sampling functions, without overlapping or double-sampling.
Based on this sub-sampling function of matrix, we further propose a sub-sampling function of four-dimensional tensor.
It is basically the same operator but is conducted on the plane of two designated dimension out of four.

\begin{figure}[tp]
    \centering
    \includegraphics[width=0.8\linewidth]{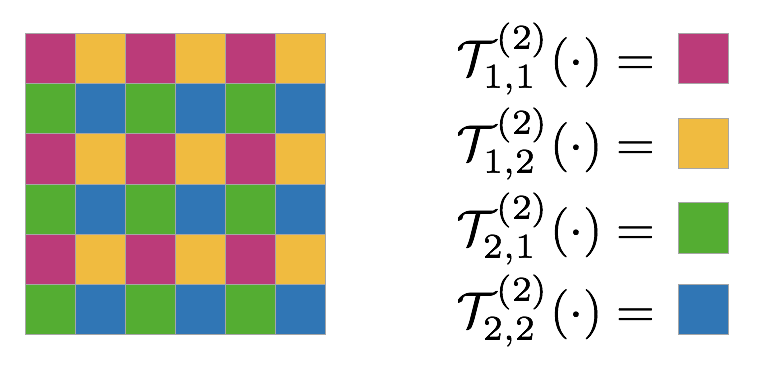}
    \caption{An example of $\Tc_{m, n}^{(2)}$ for $m,n \in\{1, 2\}$. In this case, each sampling function will return a $3\times3$ matrix. Left visualizes the input matrix whose colors denote which sampling function will select it.}
    \label{fig:ex-S}
\end{figure}

\begin{definition}
    \label{def:sub-sampling-tensor}
    (Sampling function of tensor)
    Define $\Sc_{p, q, s}^{i, j}(\cdot)$ as a function sampling the elements of a four-dimension tensor.
    $(i, j)$ denotes the dimension indexes forming a plane on which elements are sampled.
    $(p, q)$ denotes the upper-left position for sub-sampling.
    $s$ denotes the sub-sampling stride size.
\end{definition}

Note that if we treat a matrix as a four-dimensional tensor with size one in third and forth dimension, then $\Tc_{m,n}^{(s)}(\cdot)$ can be represented equivalently as $\Sc_{m, n, s}^{1, 2}(\cdot)$.

After Adding sub-sampling functions to our mathematic arsenal, we are now able to give and prove another property of convolutional tensors, which is of significant importance in the future analysis of strides.

\begin{property}
    \label{prop:double sampling}
    For any positive integer $p, q, m, n \le s$, the sub-sampled tensor
    \begin{equation}
        \Sc_{p, q, s}^{3, 4}\big(\Sc_{m, n, s}^{1, 2}\big(\Hc\big)\big)
    \end{equation}
    is still a convolutional tensor if $\Hc$ is a convolutional tensor.
    Further, this sub-sampled tensor covers all shifted version of a sub-sampled filter
    \begin{equation}
        \Tc_{p, q}^{(s)}\Big(\Oc_{m-1, n-1}\big(\Hv\big)\Big),
    \end{equation}
    where $\Oc_{m,n}(\cdot)$ denotes a padding operation which adds $m$ zero-padding along the first dimension and $n$ zero-padding along the second dimension.
\end{property}
\begin{proof}
    From the Definition~\ref{def:sub-sampling-tensor}, it is trivial to show that the $(i, j, k, \ell)$-th element in $\Sc_{p, q, s}^{3, 4}\big(\Sc_{m, n, s}^{1, 2}\big(\Hc\big)\big)$ equals to~$\Hc_{(i-1)s+m, (j-1)s+n, (k-1)s+p, (\ell-1)s+q}$.

    Since $\Hc$ is a convolutional tensor, from Property~\ref{propty:ele}, it is implied that, for any valid $i, j, k, \ell$, 
    \begin{equation}
        \begin{aligned}
            &\Big[\Sc_{p, q, s}^{3, 4}\big(\Sc_{m, n, s}^{1, 2}\big(\Hc\big)\big)\Big]_{i+1, j, k+1, \ell} \\
            = &\Hc_{(i-1)s+m+s, (j-1)s+n, (k-1)s+p+s, (\ell-1)s+q}\\
            = &\Hc_{(i-1)s+m, (j-1)s+n, (k-1)s+p, (\ell-1)s+q}\\
            = &\Big[\Sc_{p, q, s}^{3, 4}\big(\Sc_{m, n, s}^{1, 2}\big(\Hc\big)\big)\Big]_{i, j, k, \ell}
        \end{aligned}
    \end{equation}
    and, similarly, 
    \begin{equation}
        \begin{aligned}
            &\Big[\Sc_{p, q, s}^{3, 4}\big(\Sc_{m, n, s}^{1, 2}\big(\Hc\big)\big)\Big]_{i, j+1, k, \ell+1} \\
            = &\Hc_{(i-1)s+m, (j-1)s+n+s, (k-1)s+p, (\ell-1)s+q+s}\\
            = &\Hc_{(i-1)s+m, (j-1)s+n, (k-1)s+p, (\ell-1)s+q}\\
            = &\Big[\Sc_{p, q, s}^{3, 4}\big(\Sc_{m, n, s}^{1, 2}\big(\Hc\big)\big)\Big]_{i, j, k, \ell}.
        \end{aligned}
    \end{equation}

    From Property~\ref{propty:conv}, it is indicated that the sub-sampled tensor $\Sc_{p, q, s}^{3, 4}\big(\Sc_{m, n, s}^{1, 2}\big(\Hc\big)\big)$ must be a convolutional tensor.
    As we know the left-upper active area on $\Hc(1, 1, :, :)$ is the corresponding filter to $\Hc$, the filter corersponding to the sub-sampled convolutional tensor is~$\Tc_{p, q}^{(s)}\Big(\Oc_{m-1,n-1}\big(\Hv\big)\Big)$.
\end{proof}

\section{Eliminating stride from Convolution}
Armed with convolutional tensor and sub-sampling functions, we now demonstrate that in a two-dimensional convolution, non-unity stride can be easily reduced to unity stride by rearranging the element positions in images and image filters respectively.
\begin{theorem}
    \label{thm:2d striding}
    Convolving an image filter $\Hv\in\RR^{a\times b}$ with an image $\Xv\in\RR^{c\times d}$ in stride $s$ equals to the summation of regular convolutions of filters $\Hv_{p, q}$ and images $\Xv_{p, q}$ for $p, q = 1, ..., s$.
    Each $\Hv_{p, q}$ is sampled from $\Hv$ and each $\Xv_{p, q}$ is sampled from $\Xv$.
    Formally,
    \begin{equation}
        \Hv*_s\Xv = \sum_{p=1}^s \sum_{q=1}^s \Hv_{p, q} * \Xv_{p, q},
    \end{equation}
    where operator $*_s$ denotes convolution with stride $s$.
\end{theorem}
\begin{proof}
    From the definition of stride, it is clear that one can consider convolution with stride size larger than one as convolution followed by a sub-sampling function.
    Formally,
    \begin{equation}
        \Hv *_s \Xv = \Tc_{1, 1}^{(s)} \big(\Hv * \Xv\big).
        \label{eq:reduce stride}
    \end{equation}

    From Lemma~\ref{lem:sampling to inside} below, it is implied that
    \begin{equation}
        \begin{aligned}
            \Tc_{1, 1}^{(s)} \big(\Hv * \Xv\big) &= \sum_{p=1}^s \sum_{q=1}^s \Tc_{p, q}^{(s)}\big(\Hv\big) * \Tc_{p, q}^{(s)}\big(\Xv\big)\\
            \label{eq:conv}
        \end{aligned}
    \end{equation}
    By defining
    \begin{equation}
        \Hv_{p, q} = \Tc_{p, q}^{(s)}\big(\Hv\big), \quad \Xv_{p, q} = \Tc_{p, q}^{(s)}\big(\Xv\big),
    \end{equation}
    Equation~\ref{eq:reduce stride} is proven.
\end{proof}

If one considers each sub-sampled feature map $\Tc_{p, q}^{(s)}\big(\Xv\big)$ as a channel, then the summation of convolutions in Equation~\ref{eq:conv} can be reinterpreted as a multi-channel convolution.
In this sense, Theorem~\ref{thm:2d striding} claims that a multi-stride convolution can always be represented equivalently by a multi-channel convolution.
This insight inspires us to explore the simplicity of a CNN, hoping to find a more elegant and general architecture using solely multi-channel convolution with unity stride in each layer.
We will show in the next section that this non-striding architecture exists for each feedforward all convolutional net and is mathematically proven to be able to achieve comparable results if not better.

Additionally, Theorem~\ref{thm:2d striding} leads to some empirical benefits.
It potentially improves computational efficiency in a practical embedding when estimating non-unity strided convolution with larger filters by multi-channel convolution with smaller filters.
We pause the analysis to the reasons for this computational efficiency, since it involves the low-level structure of data storage in memory, which is out of the interest of our paper.

Finally, this insight potentially offer a novel perspective to answer the question in signal processing: 
Why is stride commonly used in convolutional neural network but rarely found in convolutional sparse coding?
That is caused by the equivalence between multi-stride convolution and multi-channel convolution.
Due to the common usage of multi-channel convolution in convolutional sparse coding, using non-unity stride barely bring any theoretical benefits.
A more detailed analysis is out of the scope of our paper and is considered as the future work.

\begin{lemma}
    \label{lem:sampling to inside}
    For any image filter $\Hv$ and image $\Xv$, it must be true that
    \begin{equation}
        \Tc_{m,n}^{(s)}\big(\Hv * \Xv\big) = \sum_{p=1}^s \sum_{q=1}^s \Tc_{p, q}^{(s)}\Big(\Oc_{m-1, n-1}\big(\Hv\big)\Big)*\Tc_{p, q}^{(s)}\big(\Xv\big)
        \label{eq:lemma}
    \end{equation}
\end{lemma}
\begin{proof}
    From Definition~\ref{def:conv-tensor}, it is known that
    \begin{equation}
        \Hv * \Xv = \Hc \Xv.
    \end{equation}
    Further, from Definition~\ref{def:sub-sampling-tensor} and~\ref{def:sub-sampling-matrix}, it is implied that
    \begin{equation}
        \Tc_{m, n}^{(s)}\big(\Hv * \Xv) = \Sc_{m, n, s}^{1, 2}\big(\Hc\big)\Xv.
    \end{equation}
    Since sampling functions $\Sc_{p, q, s}^{(3, 4)}(\cdot)$ for all $p, q = 1,...,s$ will cover all elements and only cover once, it is indicated that
    \begin{equation}
        \Sc_{m, n, s}^{1, 2}\big(\Hc\big)\Xv = \sum_{p=1}^s \sum_{q=1}^s \Sc_{p, q, s}^{3, 4}\big(\Sc_{m, n, s}^{1, 2}\big(\Hc\big)\big) \Tc_{p, q}^{(s)}\big(\Xv\big).
    \end{equation}
    From the Property~\ref{prop:double sampling} above, $\Sc_{p, q, s}^{3, 4}\big(\Sc_{m, n, s}^{1, 2}\big(\Hc\big)\big)$ is a convolutional tensor and is generated by the image filter $\Tc_{p, q}^{(s)}\Big(\Oc_{m-1, n-1}\big(\Hv\big)\Big)$. 
    Therefore, due to the equivalence between convolution and tensor product, Equation~\ref{eq:lemma} is proven.
\end{proof}

\section{Eliminating Stride from CNN}
\label{sec:CNN}
Now, we propose a novel strategy to eliminate non-unity strided convolutions from CNN.
From Theorem~\ref{thm:2d striding}, we are allowed to replace a non-unity strided convolution by a unity strided convolution without loss of performance for evaluation.
However, this replacement might not be trivial in a CNN architecture as any changes of the number of channels in one layer will affect the structures of other layers.
To make this difficulty clearer and also give a direct insight, let us consider a simple three layer example.
Note that, to be able to focus on convolution and stride, we will follow~\cite{springenberg2014striving} to replace max-pooling by a convolutional layer with increased stride, and restrict ourselves to the analysis of all convolutional net in this section.

\begin{figure}[tp]
    \centering
    \subfloat[The original architecture]{\includegraphics[width=\linewidth]{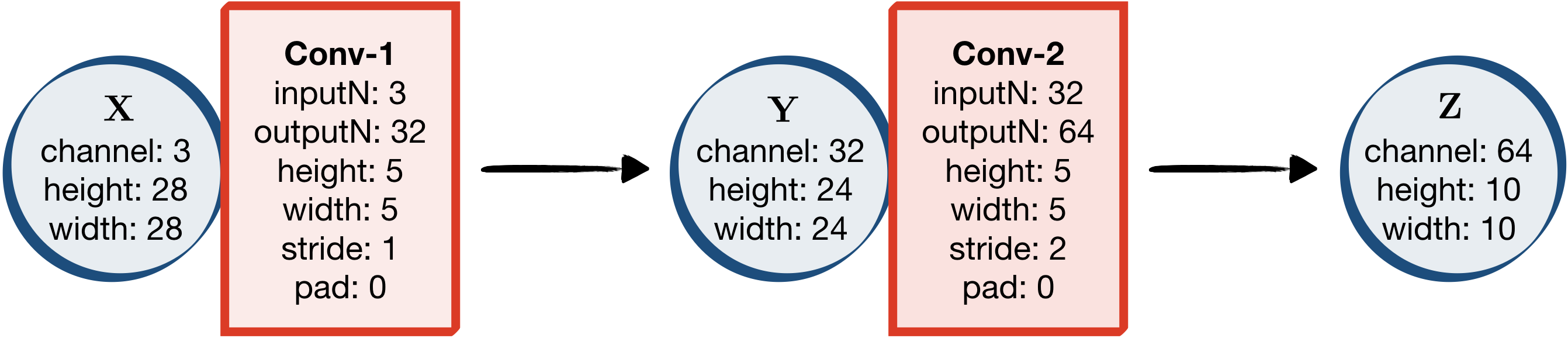}}\\
    \subfloat[Replace Conv-2 by Conv-2' and $\Yv$ by $\Yv'$. ]{\includegraphics[width=\linewidth]{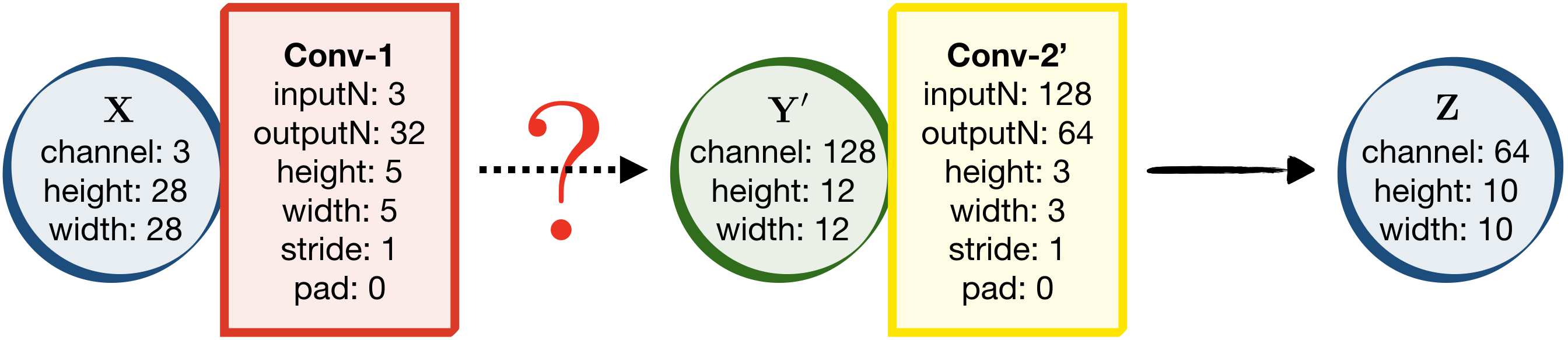}}\\
    \subfloat[Replace Conv-1 by Conv-1' and $\Xv$ by $\Xv'$.]{\includegraphics[width=\linewidth]{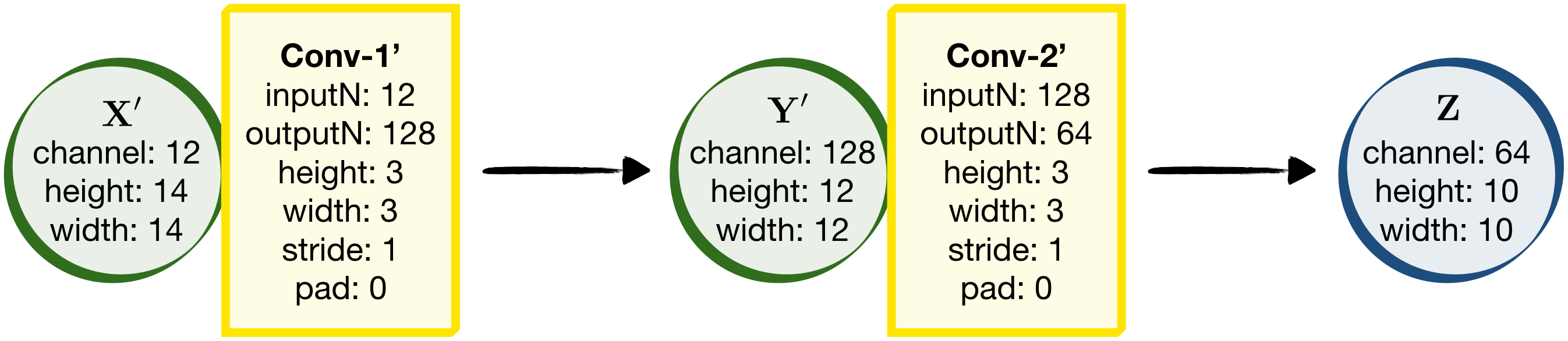}}
    \caption{A three layer CNN example. It shows that solely utilizing Theorem~\ref{thm:2d striding} is not sufficient to eliminate multi-stride from an all convolutional net.
    how to derive a non-striding architecture from a toy all convolutional net and also visualize the utility of Theorem~\ref{thm:2d striding} and~\ref{thm: all-conv-net}.}
    \label{fig:ex-3layer}
\end{figure}

Suppose that we have a three layer CNN whose architecture is shown in Figure~\ref{fig:ex-3layer}~(a).
To eliminate non-unity strided convolutions from this example, one can replace the convolution layer Conv-2 by a new convolution layer Conv-2' which utilizes unity stride and $\Yv$ by $\Yv'$ which rearranges the element positions of $\Yv$.
From Theorem~\ref{thm:2d striding}, it is guaranteed that this replacement will hold the same output as before.
The new architecture is shown in Figure~\ref{fig:ex-3layer}~(b).
However, one may notice that this replacement ruins the connection from the previous layer Conv-1 to $\Yv'$, resulting in a failure of forward pass.

To solve the problem and reconnect the previous layer to the modified current layer, we need to revisit Lemma~\ref{lem:sampling to inside}.
As one observes, each channel of $\Yv'$ is a certain sub-sampling of the feature map in a channel of $\Yv$.
Further each channel of $\Yv$ is the result of convolving $\Xv$ with associated filters in Conv-1.
Therefore, from Lemma~\ref{lem:sampling to inside}, it is implied that each channel of $\Yv'$ is actually the summation of convolving all possible sub-sampled $\Xv$ with sub-sampled filters in Conv-1.
In this perspective, to connect two layers, one can replace Conv-1 by a multi-channel convolution layer, Conv-1', whose filter has more channels but in smaller size, and at the same time $\Xv$ by $\Xv'$, where each channel of $\Xv'$ is a certain sub-sampling of a channel of $\Xv$.
Since again the sub-sampling function will cover and only cover all elements once, $\Xv'$ actually rearranges the element positions of~$\Xv$, sharing the same number and values of elements.

By this strategy, we derived a new architecture shown in Figure~\ref{fig:ex-3layer}~(c) that only utilizes convolutions with unity stride and are guaranteed to have a complete forward pass.
More importantly, this new architecture also hold the equivalence to the original multi-stride CNN in terms of the evaluations of forward pass, guaranteed by Lemma~\ref{lem:sampling to inside}.
To offer a deeper insight and prove this procedure more rigorously, we next propose and prove Theorem~\ref{thm: all-conv-net}, summarizing the strategy above and generalize it to additional layers.

\begin{theorem}
    \label{thm: all-conv-net}
    Any all convolutional net can be simplified to a non-striding architecture where only unity stride is utilized by each convolution layer, such that these two architectures are mathematically equivalent with respect to the evaluation of forward pass.
\end{theorem}
\begin{proof}
%
    Suppose we have an all convolutional net with $L$ layers, where each layer has $C_\ell$ channels and utilizes $s_\ell$ as stride size.
    Denote $\Xv^{(\ell, c)}$ as the feature map in $\ell$-th layer $c$-th channel and $\Hv^{(\ell, c)}$ is the corresponding filter.
    We now use mathematical induction to prove this theorem.

    \textbf{The last layer (basis):}
    From the architecture, we have
    \begin{equation}
        \Xv^{(L, c)} = \eta\{\sum_{k=1}^{C_{L-1}} \Hv^{(L-1, k)} *_{s_{L-1}} \Xv^{(L-1, k)}\},
    \end{equation}
    where function~$\eta(\cdot)$ is an element-wise non-linear activation function such as Rectified Linear Unit(ReLU).
    From Theorem~\ref{thm:2d striding}, it is derived that
    \begin{equation}
        \label{eq:basis}
        \Xv^{(L, c)} = \eta\{\sum_{k=1}^{C_{L-1}} \sum_{p=1}^{s_{L-1}} \sum_{q=1}^{s_{L-1}} \Hv^{(L-1, k)}_{p, q} * \Xv^{(L-1, k)}_{p, q}\},
    \end{equation}
    where
    \begin{equation}
        \Xv_{p, q}^{(L-1, k)} = \Tc_{p, q}^{(s_{L-1})}\big(\Xv^{(L-1, k)}\big),
    \end{equation}
    \begin{equation}
        \Hv_{p, q}^{(L-1, k)} = \Tc_{p, q}^{(s_{L-1})}\big(\Hv^{(L-1, k)}\big).
        \label{eq: Hv-last}
    \end{equation}

    For convenience, let us define $\tilde{\Xv}^{(L-1)}$ as a feature map with $C_{L-1}\cdot s_{L-1} \cdot s_{L-1}$ channels.
    Each channel of it is $\Xv_{p, q}^{(L-1, k)}$ for $p, q=1,...,s_{L-1}$ and $k=1,...,C_{L-1}$.
    Similarly, define $\tilde{\Hv}^{(L-1, c)}$ as a filter with $C_{L-1}\cdot s_{L-1}\cdot s_{L-1}$ channels.
    Each channel of it is $\Hv_{p, q}^{(L-1, k)}$ for $p, q=1,...,s_{L-1}$ and $k=1,...,C_{L-1}$.
    Then the right hand side of Equation~\ref{eq:basis} can be reinterpreted as a multi-channel convolution between $\tilde{\Xv}^{(L-1)}$ and $\tilde{\Hv}^{(L-1, c)}$. Formally,
    \begin{equation}
        \Xv^{(L, c)} = \eta\{\tilde{\Xv}^{(L-1)} * \tilde{\Hv}^{(L-1, c)}\}.
    \end{equation}

    This basis step shows that one can always replace the last convolution layer by a multi-channel convolution with unity stride, such that the last layer feature map $\Xv^{(L, c)}$ stay the same.
    Next, we will show how to represent each channel of $\tilde{\Xv}^{(L-1)}$ by $\Xv^{(L-2)}$ and $\Hv^{(L-2)}$.

    \textbf{Intermediate layer (inductive step):}
    Let us consider the $\ell$-th layer, where $\ell < L-1$.
    Given any valid positive integer $c \le C_{l+1}, m \le s, n \le s$ and $s$, one can show that a sub-sampling of activation layer is equivalent to sub-sampling before the activation layer as long as the activation layer is element-wise.
    Formally,
    \begin{eqnarray}
        &&\Tc_{m, n}^{(s)}\big(\Xv^{(\ell+1, c)}\big)\nonumber\\
        &= &\Tc_{m,n}^{(s)}\Big(\eta\{\sum_{k=1}^{C_\ell} \Hv^{(\ell, k)} *_{s_\ell} \Xv^{(\ell, k)} \}\Big) \nonumber\\
        &= &\eta\Big\{\Tc_{m,n}^{(s)}\big(\sum_{k=1}^{C_\ell} \Hv^{(\ell, k)} *_{s_\ell} \Xv^{(\ell, k)} \big)\Big\}.
    \end{eqnarray}

    Further, from the definition of stride, one can replace the multi-stride convolution by a single-stride convolution followed by a sub-sampling function. Formally,
    \begin{eqnarray}
        &&\Tc_{m, n}^{(s)}\big(\Xv^{(\ell+1, c)}\big)\nonumber\\
        &=&\eta\Big\{\Tc_{m,n}^{(s)}\big(\sum_{k=1}^{C_\ell} \Hv^{(\ell, k)} *_{s_\ell} \Xv^{(\ell, k)} \big)\Big\} \nonumber \\
        &= &\eta\Big\{\Tc_{m,n}^{(s)}\Big(\sum_{k=1}^{C_\ell} \Tc_{1, 1}^{(s_\ell)} \big(\Hv^{(\ell, k)} * \Xv^{(\ell, k)}\big) \Big)\Big\}\nonumber \\
          &= &\eta\Big\{\sum_{k=1}^{C_\ell} \Tc_{m,n}^{(s)}\Big(\Tc_{1, 1}^{(s_\ell)} \big(\Hv^{(\ell, k)} * \Xv^{(\ell, k)}\big) \Big)\Big\}
        \label{eq:inter}
    \end{eqnarray}

    It is clear to see that the composition of two sub-sampling function is also a sub-sampling function:
    \begin{equation}
        \Tc_{m,n}^{(s)} \circ \Tc_{1, 1}^{(s_\ell)} = \Tc_{m', n'}^{(s')},
    \end{equation}
    where, for the sake of brevity, we omit the derivation but give the results:
    \begin{equation}
        m' = (m-1)s_\ell +1, ~ n'=(n-1)s_\ell+1, ~ s' = ss_\ell.
    \end{equation}
    Therefore, by combining the two sub-sampling functions, Equation~\ref{eq:inter} can be simplified as
    \begin{equation}
        \Tc_{m, n}^{(s)}\big(\Xv^{(\ell+1, c)}\big) = \eta\Big\{\sum_{k=1}^{C_\ell} \Tc_{m', n'}^{(s')} \big(\Hv^{(\ell, k)} * \Xv^{(\ell, k)}\big)\Big\}.
    \end{equation}
    From the Lemma~\ref{lem:sampling to inside}, it is implied that
    \begin{equation}
        \Tc_{m, n}^{(s)}\big(\Xv^{(\ell+1, c)}\big) = \eta\Big\{\sum_{k=1}^{C_\ell} \sum_{p=1}^{s'} \sum_{q=1}^{s'} \Hv^{(\ell, k)}_{p,q, m',n'} * \Xv^{(\ell, k)}_{p, q}\Big\},
        \label{eq:rearrange}
    \end{equation}
    where
    \begin{equation}
        \Hv^{(\ell, k)}_{p,q, m',n'} = \Tc_{p, q}^{(s')} \Big(\Oc_{m'-1, n'-1}\big(\Hv^{(\ell, k)}\big)\Big),
        \label{eq: Hv}
    \end{equation}
    \begin{equation}
        \Xv^{(\ell, k)}_{p, q} = \Tc_{p, q}^{(s')}\big(\Xv^{(\ell, k)}\big).
        \label{eq: Xv}
    \end{equation}

    Let us define $\tilde{\Xv}^{(\ell)}$ as a feature map with $C_{\ell}\cdot s' \cdot s'$ channels.
    Each channel of it is $\Xv_{p, q}^{(\ell, k)}$ for $p, q=1,...,s'$ and $k=1,...,C_{\ell}$.
    Similarly, define $\tilde{\Hv}^{(\ell, c')}$ as a filter with $C_{\ell}\cdot  s' \cdot s'$ channels, where $c'$ depends on $(m, n, c)$.
    Each channel of it is $\Hv_{p, q, m',n'}^{(\ell, k)}$ for $p, q=1,...,s_{L-1}$ and $k=1,...,C_{L-1}$.
    Then the left hand side of Equation~\ref{eq:rearrange} can be considered as the $c'$-th channel of $\tilde{\Xv}^{(\ell+1)}$.
    The right hand side of Equation~\ref{eq:rearrange} can be reinterpreted as a multi-channel convolution between $\tilde{\Xv}^{(\ell)}$ and $\tilde{\Hv}^{(\ell, c')}$. Formally,
    \begin{equation}
        \tilde{\Xv}^{(\ell+1, c')} = \eta\big\{\tilde{\Xv}^{(\ell)} * \tilde{\Hv}^{(\ell, c')}\big\}.
    \end{equation}

    This inductive step shows that one can always represent each channel of $\tilde{\Xv}^{(\ell+1)}$ as the result of a multi-channel convolution with unity stride of $\tilde{\Xv}^{(\ell)}$ and a multi-channel filter.


    \textbf{Mathematical Induction:}
    In summary, in an all convolutional net, the followings are proven true:
    \begin{itemize}
        \item The last layer feature map $\Xv^{(L)}$ can always be represented by $\tilde{\Xv}^{(L-1)}$ with a unity strided convolution followed by a non-linear function.
        \item For any $\ell < L-1$, the feature map $\tilde{\Xv}^{(\ell+1)}$ can always be represented by $\tilde{\Xv}^{(\ell)}$ with a unity strided convolution followed by a non-linear function.
        \item Since sub-sampling function $\Tc_{m,n}^{(s)}$ covers and only covers once all elements, the feature map $\tilde{\Xv}^{(\ell)}$ is a certain reshape of $\Xv^{(\ell)}$.
    \end{itemize}

    From mathematical induction, it is implied that, for any all convolutional net, one can reshape the feature map sequentially from the last layer to the first layer and thus derive a new architecture where only unity stride is utilized by each convolution layers.
    As a result of more elegant structures without non-unity stride, we refer to this new architecture as the non-striding architecture.
    Moreover, because the equality always holds during reshaping, shown in Equation~\ref{eq:basis} and Equation~\ref{eq:rearrange}, the non-striding architecture must be equivalent to the original one with respect to the output of forward pass.
\end{proof}

\section{Discussion and Implication}
The theory described so far guarantees the existence of the non-striding architecture for any all convolutional net and the proof of Theorem~\ref{thm: all-conv-net} further provides a methodology to recover it.
However, one might wonder what we can benefit from eliminating stride.
In this section, we will compare in detail the non-striding architecture against the original network, further establish theoretical and practical benefits, and also note a few practical limitations of our framework.

One notable difference between the non-striding and original architecture is the dimension of images fed into the network.
The non-striding architecture is required to take $\tilde{\Xv}^{(0)}$ as input which is a certain reshape of normal images $\Xv^{(0)}$.
This step can be conducted before training or evaluating the neural network by prepocessing images.
Thus it potentially transfers some computation from neural network to preprocessing step.

Another significant difference is the size of filters between these two architectures.
For example, in the architecture shown in Figure~\ref{fig:ex-3layer}, one can see that the size of filter in Conv-1 is $3\times 32\times 5\times 5$, however in Conv-1' it is $12\times 128\times 3\times 3$.
More generally, by denoting $H, W$ as the height and width of the filter and following the notations in Section~\ref{sec:CNN}, the size of $\Hv^{(\ell)}$ is $C_{\ell}\cdot C_{\ell+1}\cdot W_{\ell} \cdot H_{\ell}$, while the size of $\tilde{\Hv}^{(\ell)}$ is $s^2\cdot C_{\ell}\cdot C_{\ell+1}\cdot W_{\ell} \cdot H_{\ell}$.
This is because, in $\tilde{\Hv}^{(\ell)}$, the number of input channel becomes $(s')^2$ more and the number of output channel becomes $s^2$ more, but filter size only becomes $(s')^2$ smaller.
This increased size of filters indicates that the non-striding architecture has higher capacity than the original one.
Even though this augmented capacity is potentially to increase the performance, it also possibly ruins the training due to too large searching area.

The reason why this additional capacity is not observed during forward passing a pre-trained model is that the strategy we proposed naturally forces some parameter sharing in the non-striding architecture.
Specifically, in Equation~\ref{eq: Hv}, it is clear that filters $\Hv_{p, q, m',n'}^{(\ell, k)}$ with different $m'$ and $n'$ are forced to share parameters since they are all sampled from the same filters $\Hv^{(\ell, k)}$.
This finding unveils an important role of striding in a CNN, that is striding is an efficient practice to force parameters shared among different channels.
Specifically, one can observe that using non-unity stride in the original architecture is equivalent to sharing parameters among different channels in the non-striding architecture.
This decrease of degree of freedom leads to a lower expressibility of network but higher possibility of searching a good local minimum.
This makes striding still an invaluable practice to help training in an empirical usage.

The final contribution of our work is to pursuit the simplest architecture of a CNN.
Following the work~\cite{springenberg2014striving} suggesting to replace max-pooling by convolution with larger stride, our work proposes to replace non-unity strided convolutions by unity strided convolutions without loss of performance.
This further simplifies the neural network, leading to a more elegant architecture with only non-linearity and classical convolution without the need of designing stride size.
We hope this non-striding architecture shall facilitate and simplify the future theoretical analysis of CNNs.

\begin{table*}[]
    \centering
    \begin{tabular}{|c|c|c|c|c|c|c|c|c|}
        \hline
        \rowcolor[HTML]{86b1f5} 
        \textbf{LeNet} & \multicolumn{6}{c|}{\cellcolor[HTML]{86b1f5}\textbf{Architecture}}  & \textbf{Accuracy} & \textbf{Training Time} \\ \hline
    \cellcolor[HTML]{c3d9fb}Non-Unity Stride & \begin{tabular}{@{}c@{}}Input image \\ $1\times 28\times 28$\end{tabular} & \begin{tabular}{@{}c@{}}$5\times5$ \\ conv. 20 \end{tabular} & \begin{tabular}{@{}c@{}}$2\times2$ \\ conv. 20 \\\textbf{stride 2}\end{tabular} & \begin{tabular}{@{}c@{}} $5\times5$ \\ conv. 50 \end{tabular} & \begin{tabular}{@{}c@{}} $2\times2$ \\ conv. 50 \\ \textbf{stride 2}\end{tabular} & \begin{tabular}{@{}c@{}} fc. 500 \\ ReLU \end{tabular}& $98.17\%$ & $451.91$ sec\\ \hline
        \cellcolor[HTML]{c3d9fb}Unity Stride     & \begin{tabular}{@{}c@{}}Input image \\ $16\times 7\times 7$\end{tabular} & \begin{tabular}{@{}c@{}}$2\times2$ \\ conv. 320 \end{tabular} & \begin{tabular}{@{}c@{}}$1\times1$ \\ conv. 80 \end{tabular} & \begin{tabular}{@{}c@{}} $3\times3$ \\ conv. 200 \end{tabular} & \begin{tabular}{@{}c@{}} $1\times1$ \\ conv. 50 \end{tabular} & \begin{tabular}{@{}c@{}} fc. 500 \\ ReLU \end{tabular} & $98.23\%$ & $680.97$ sec \\ \hline
    \end{tabular}
    \caption{Model description of the two networks derived from LeNet. ``Non-Unity Stride'' denotes the architecture using non-unity strided convolution layer, while ``unity stride'' denotes the architecture using unity strided convolution layer. ``Accuracy'' refers to as the testing accuracy after independently training two architectures. As an example, ``$5\times5$ conv. 20'' denotes as a convolution layer whose kernel size is $5\times5$ and the number of output channel is 20. }
    \label{tab:exp-arc}
    
    \vspace{8mm}

    \begin{tabular}{|c|c|c|c|c|c|c|}
        \hline
        \rowcolor[HTML]{c6d7c1} 
        \textbf{LeNet} & \textbf{Input Image} &\multicolumn{5}{c|}{\cellcolor[HTML]{c6d7c1}\textbf{Feature Map}}\\ \hline
    \cellcolor[HTML]{e3ece1}Non-Unity Stride & $1\times 28\times 28$ & $20\times24\times24$ & $20\times12\times12$ & $50\times8\times8$ & $50\times4\times4$ & 500\\ \hline
        \cellcolor[HTML]{e3ece1}Unity Stride     & $16\times 7\times 7$ & $320\times6\times6$ & $80\times6\times6$ & $200\times4\times4$ & $50\times4\times4$ & 500\\ \hline
    \end{tabular}
    \caption{Listing the dimension of each feature maps in the two architectures, utilizing channel$\times$height$\times$width format.}
    \label{tab:exp-var}
\end{table*}

\section{A Toy Experiment}
The theory described so far is very general and provides a guarantee to the existence of the non-striding architecture to arbitrary all convolutional net.
To verify the theory and also provide a sense from a concrete neural network, we conducted a toy experiment applying a modified LeNet~\cite{lecun1998gradient} to MNIST dataset.
LeNet is a perfect target in our case due to its small scale, leading to a clear conclusion, quick verification, and most importantly being able to restrict the non-striding architecture to a similar scale to the original one.

Following the work~\cite{springenberg2014striving}, we first replace all pooling layers in LeNet by a convolution with stride size two.
The architecture is shown in the Table~\ref{tab:exp-arc}, named non-unity stride.
We then utilize the proposed strategy to eliminate non-unity strided convolutions by unity strided convolutions, resulting to a non-striding architecture, shown in the Tabel~\ref{tab:exp-arc}, named unity stride.
As one can see, non-striding architecture utilizes smaller filters but with more channels, as expected and indicated by Theorem~\ref{thm:2d striding}.
Additionally, for the convenience of analyzing feature maps, we list the dimension of each feature map in Table~\ref{tab:exp-var}.
As observed, in each layer, the length (not the shape) of the feature map stays the same in both architectures, as expected and indicated by Theorem~\ref{thm: all-conv-net}.
Recall that $\tilde{\Xv}^{(\ell)}$ is a certain reshape of $\Xv^{(\ell)}$.

To compare these two architectures in terms of performance, we conducted over two different settings.
In the first setting, we only train the network with striding and then copy the learned parameters to the non-striding architecture.
In this case, the performance of these two architectures are exactly the same, as expected and proven by Theorem~\ref{thm: all-conv-net}; the results are therefore omitted.
In the other setting, we train both of these two architectures from scratch without sharing any parameters.
The accuracy and time consumed for training are listed in Table~\ref{tab:exp-arc}.
One can clearly see from the table that the non-striding architecture achieves a higher accuracy but consuming more time during training, as expected.
Recall that the non-striding architecture has a larger parameter space, resulting in an increase of capacity.

\section{Conclusion}
This paper, to our best knowledge, at the first time proposed to represent a non-unity strided convolution by a unity strided convolution with multi-channels.
Based on this insight, we demonstrated that in any all convolutional net, non-unity strided convolution can be replaced by unity strided convolution without loss of performance for evaluations.
Hereby, any feedfoward CNN are proven to have a mathematically equivalent non-striding CNN architecture during evaluation such that it consists only of non-linearity and regular convolution with solely unity stride.
We hope this non-striding architecture shall facilitate and simplify the future theoretical analysis of CNNs.
Finally, by observing an increase in the capacity of the non-striding architecture, we demonstrated that striding reduces the number of trainable variables in convolutional filters by sharing parameters.
This finding makes striding still a useful practice when designing CNN architectures.

{\small
\bibliographystyle{ieee}
\bibliography{egbib}
}

\end{document}